\documentclass[letterpaper]{article} 
\usepackage{aaai25}  
\usepackage{times}  
\usepackage{helvet}  
\usepackage{courier}  
\usepackage[hyphens]{url}  
\usepackage{graphicx} 
\usepackage{natbib}  
\usepackage{caption} 
\usepackage{algorithm}
\usepackage{subcaption}

\usepackage{newfloat}
\usepackage{listings}
\DeclareCaptionStyle{ruled}{labelfont=normalfont,labelsep=colon,strut=off} 
\floatstyle{ruled}
\newfloat{listing}{tb}{lst}{}
\floatname{listing}{Listing}
\usepackage{xcolor,soul,framed} 

\colorlet{shadecolor}{yellow}
\usepackage{multirow}
\usepackage{array}
\usepackage{amsthm}

\newtheoremstyle{theorem}
  {\topsep}
  {\topsep}
  {}
  {}
  {\itshape}
  {:}
  {.5em}
  {\thmname{#1}\thmnumber{ #2}\thmnote{ (#3)}}
\theoremstyle{theorem}

\newtheoremstyle{proposition}
  {\topsep}
  {\topsep}
  {}
  {}
  {\itshape}
  {:}
  {.5em}
  {\thmname{#1}\thmnumber{ #2}\thmnote{ (#3)}}
\theoremstyle{proposition}

\newtheorem{theorem}{Theorem}

\newtheorem{corollary}{Corollary}

\usepackage{mdwmath}
\usepackage{mdwtab}
\usepackage{eqparbox}
\usepackage{url}

\usepackage{amssymb}
\usepackage{amsfonts}
\usepackage{amsmath}
\usepackage{multicol}
\usepackage{bm}

\usepackage{algorithm}
\usepackage{algorithmicx}
\usepackage{algpseudocode}
\usepackage{cancel}
\usepackage{subeqnarray}
\usepackage{cases}
\usepackage{mathrsfs}
\usepackage{latexsym}
\usepackage{booktabs}
\usepackage{enumitem}

\usepackage[noabbrev,nameinlink]{cleveref}

\crefname{figure}{Figure}{Figures}
\Crefname{figure}{Figure}{Figures}
\crefname{subfigure}{Figure}{Figures}
\Crefname{subfigure}{Figure}{Figures}

\makeatletter

\newcommand{\Rmnum}[1]{\expandafter\@slowromancap\romannumeral #1@}
\makeatother

\DeclareCaptionStyle{ruled}{labelfont=normalfont,labelsep=colon,strut=off} 
\floatstyle{ruled}
\newfloat{listing}{tb}{lst}{}
\floatname{listing}{Listing}
\title{Graph Agent Network: Empowering Nodes with Inference Capabilities for Adversarial Resilience}
\author{
    Ao Liu\textsuperscript{\rm 1},
    Wenshan Li\textsuperscript{\rm 2}\thanks{Corresponding author},
    Tao Li\textsuperscript{\rm 1},
    Beibei Li\textsuperscript{\rm 1},
    Guangquan Xu\textsuperscript{\rm 3,4}, \\
    Pan Zhou\textsuperscript{\rm 5},
    Wengang Ma\textsuperscript{\rm 1},
    Hanyuan Huang\textsuperscript{\rm 1}\\
}
\affiliations{
    \textsuperscript{\rm 1}School of Cyber Science and Engineering, Sichuan University \\
    \textsuperscript{\rm 2}School of Cyber Science and Engineering, Chengdu University of Information Technology \\
    \textsuperscript{\rm 3}College of Information Science and Technology, Shihezi University \\
    \textsuperscript{\rm 4}Tianjin Key Laboratory of Advanced Networking (TANK), College of Intelligence and Computing, Tianjin University \\
    \textsuperscript{\rm 5}School of Cyber Science and Engineering, Huazhong University of Science and Technology \\
    
    \{aliu, litao, libeibei, mawengang941206\}@scu.edu.cn,
    helenali@cuit.edu.cn,
    losin@tju.edu.cn, 
    panzhou@hust.edu.cn,
    huanghanyuan@stu.scu.edu.cn
}

\begin{document}

\maketitle

\begin{abstract}
End-to-end training with global optimization have popularized graph neural networks (GNNs) for node classification, yet inadvertently introduced vulnerabilities to adversarial edge-perturbing attacks. Adversaries can exploit the inherent opened interfaces of GNNs' input and output, perturbing critical edges and thus manipulating the classification results. Current defenses, due to their persistent utilization of global-optimization-based end-to-end training schemes, inherently encapsulate the vulnerabilities of GNNs. This is specifically evidenced in their inability to defend against targeted secondary attacks. In this paper, we propose the Graph Agent Network (GAgN) to address the aforementioned vulnerabilities of GNNs. GAgN is a graph-structured agent network in which each node is designed as an 1-hop-view agent. Through the decentralized interactions between agents, they can learn to infer global perceptions to perform tasks including inferring embeddings, degrees and neighbor relationships for given nodes. This empowers nodes to filtering adversarial edges while carrying out classification tasks. Furthermore, agents' limited view prevents malicious messages from propagating globally in GAgN, thereby resisting global-optimization-based secondary attacks. We prove that single-hidden-layer multilayer perceptrons (MLPs) are theoretically sufficient to achieve these functionalities. Experimental results show that GAgN effectively implements all its intended capabilities and, compared to state-of-the-art defenses, achieves optimal classification accuracy on the perturbed datasets.
\end{abstract}

\section{Introduction}
Graph neural networks (GNNs) have become state-of-the-art models for node classification tasks by leveraging end-to-end global training paradigms to effectively learn and extract valuable information from graph-structured data~\cite{hamilton_2017_SAGE}. However, this approach has also inadvertently introduced inherent vulnerabilities, making GNNs vulnerable to adversarial edge-perturbing attacks. These vulnerabilities arise from GNNs exposing a global end-to-end training interface, which while allowing for precise classification, also provides adversaries with opportunities to attack GNNs. These attacks poses a significant challenge to the practical deployment of GNNs in real-world applications where security and robustness are of paramount importance, leading to a critical issue in various application areas~\cite{yang2022difference}, including those where adversarial perturbations undermine public trust~\cite{Kreps2020SinceAdv}, interfere with human decision making~\cite{Walt2019NMI}, and affect human health and livelihoods~\cite{Samuel2019Sci}.

 \begin{figure}[tb]
    \centering
    \includegraphics[width=0.48\textwidth]{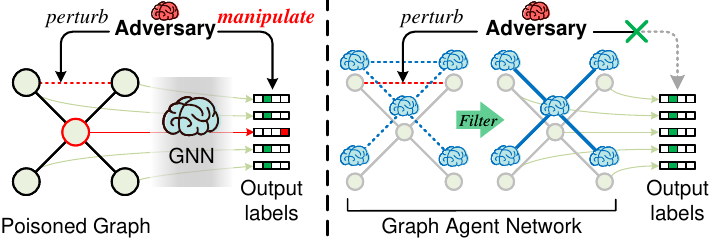}
    \caption{The principle of adversary exploiting the inherent vulnerabilities caused by global optimization to attack GNNs, and the defense mechanism of GAgN against these attacks.}
    \label{fig_intro}
\end{figure}

In response to the edge-perturbing attacks, existing defense mechanisms primarily rely on global-optimization-based defense methods~\cite{sun2022adversarial}. Their objective is to enhance the robustness of GNNs through adversarial training~\cite{feng2019GAD}, aiming to increase the tolerant ability of perturbations to defend against potential adversarial perturbations. These approaches may still inherit some inherent vulnerabilities of GNNs, as their model frameworks still follow the global-optimization pattern of GNNs. Representative examples include: (1) RGCN~\cite{Zhu2019RGCN} replaces the hidden representations of nodes in each graph convolutional network (GCN)~\cite{Kpif_2017_ICLR} layer to the Gaussian distributions, to further absorb the effects of adversarial changes. (2) GCN-SVD~\cite{entezari2020SVD} combines a singular value decomposition (SVD) filter prior to GCN to eliminate adversarial edges in the training set.
(3) STABLE~\cite{li2022STABLE} reforms the forward propagation of GCN by adding functions that randomly recover the roughly removed edges.
Unfortunately, the computational universality of GNNs has been recently demonstrated~\cite{Loukas2019GNN_DvW,xu2018Power}, signifying that attributed graphs can be classified into any given label space, even those subject to malicious manipulation.
This implies that defense approaches focused on enhancing GNNs' global robustness~\cite{zhang2020gnnguard, liu2021EGNN,li2022STABLE, liu2024towards} remain theoretically vulnerable, with potential weaknesses to secondary attacks. In fact, the vulnerability of existing global-optimization-based defenses has been theoretically proven~\cite{Liu_2022_TNNLS}. Adversaries can reduce the classification accuracy of these defenses once again by launching secondary attacks using these defenses as surrogate models. Experimental evidence demonstrates that even under the protection of state-of-the-art defensive measures~\citep{NEURIPS2020_99e314b1}, secondary attacks targeting these defenses successfully mislead 72.8\% of the nodes into making incorrect classifications once again.

To overcome the issues mentioned above, inspired by the natural filtering ability of decentralized intelligence~\cite{saldanha2022swarm}, we propose a decentralized agent network called {G}raph {Ag}ent {N}etworks (GAgN) whose principle is shown in Figure~\ref{fig_intro}. GAgN empowers nodes with autonomous awareness, while limiting their perspectives to their 1-hop neighbors. As a result, nodes no longer rely solely on global-level end-to-end training data. Instead, they progressively gain the perception of the entire network through communication with their neighbors and accomplish self-classification.

 \begin{figure*}[t]
    \centering
    \includegraphics[width=0.93\textwidth]{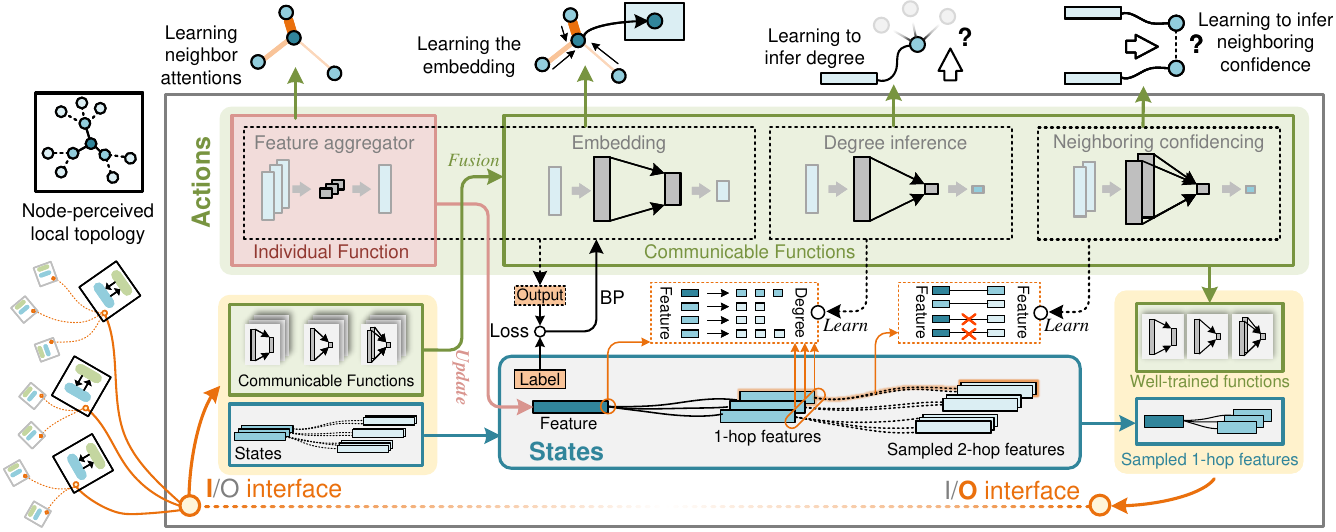}
    \caption{The internal structure and communication paradigm of an agent within GAgN.}
    \label{fig_overallFrame}
\end{figure*}

Specifically, GAgN is designed to enhance the agent-based model (ABM)~\cite{KHODABANDELU_2021_ABM}, ensuring improved compatibility with graph-based scenarios. In GAgN, agents are interconnected within the graph topology and engage in cautious communication to mitigate the impact of potential adversarial edges. As time progresses, these agents systematically exchange information to broaden their receptive fields, perceive the global information, and attain decentralized intelligence.
An agent primarily includes two major abilities: 1) Storage: states, which are storable features, and actions, which are trainable functions for inference. 2) Communication: agents can receive states and actions from neighboring agents and integrate them to improve their own inference capabilities.

During the communication round, an agent organizes multiple restricted-view mini-datasets based on information received from its immediate neighbors, and updates its states and actions accordingly. After sufficient communications, the agent accumulates enough historical experience to perform specific tasks. These tasks include: 1) computing its own embedding, 2) estimating the possible degree of a given node, and 3) determining the neighboring confidence of two given nodes. The first function enables GAgN to perform node classification, whereas the second and third functions collaboratively contribute to filtering adversarial edges.

Furthermore, we rigorously demonstrate that all these functions can be accomplished by the single-hidden-layer multilayer perceptron (MLPs). In other words, theoretically, within an agent, only three trainable matrices are required in the trainable part, excluding multi-layer networks, bias matrices, and other trainable parameters, to successfully perform the corresponding tasks. This conclusion offers theoretical backing for substantially reducing GAgN's computational complexity. In practice, by instantiating nodes as lightweight agents, a GAgN can be constructed to carry out node classification tasks with adversarial resilience.

Our main contributions including:
\begin{itemize}[left = 0.1cm]
  \item We propose a decentralized agent network, GAgN, which empowers nodes with the ability to autonomously perceive and utilize global intelligence to address the inherent vulnerabilities of GNNs and existing defense models.
  \item We theoretically prove that an agent can accomplish the relevant tasks using only three trainable matrices.
  \item We experimentally show that primary functions of GAgN have been effectively executed, which collectively yielding state-of-the-art accuracy in perturbed datasets.
\end{itemize}

\section{The Proposed Method}\label{sec_method}


We start with action spaces and discuss the state's structure and update strategies alongside other components.
Figure~\ref{fig_overallFrame} shows the workflow of GAgN.

\subsubsection{Preliminaries}
Consider connected graphs $\mathcal{G}=(\mathcal{V},\mathcal{E})$ with $|\mathcal{V}|$ nodes.
The degree of node $v_i$ is $\mathrm{deg}_i$. The feature vector and one-hot label of $v_i$ are $\mathbf{z}_i \in \mathbb{R}^{1 \times d_z}$ and $\mathbf{\ell}_i \in \mathbb{R}^{1 \times d_L}$.
The maximum degree is $\mathrm{deg}_{\mathrm{max}}$.
The neighborhood $\mathrm{N}_i$ of $v_i$ includes all $v_j$ with $e_{i,j} \in \mathcal{E}$.
The edge between $v_i$ and $v_j$ is $e_{i,j}$.
Given a set of vectors, $\mathrm{Concat}_r(\cdot)$ and $\mathrm{Concat}_c(\cdot)$ output the row-wise and column-wise concatenated matrices.

\subsection{Action Spaces}\label{Sec_action_space}

The action space consists of all potential actions a node can take in a reasoning task~\cite{peng2017learning}. In GAgN, it is designed as a set of perception functions with the same parameter structure but different specific parameters. The optimal parameters for agents in each communication round can be found within this action space. We use node $v_i$ as an example to illustrate these functions.

\subsubsection{Individual Function}

Parameters of individual functions are specific to each agent, reflecting the understanding gained through interactions with neighbors, without direct agent interactions. This includes:

$\blacktriangleright$ Neighbor feature aggregator $\mathcal{A}_i: \mathbb{R}^{(\text{deg}_i + 1) \times d_z} \to \mathbb{R}^{d_z}$ aggregates features from $v_i$’s neighborhood into $v_i$. Similar to GAT, the node’s own features are also aggregated. The trainable parameter of $\mathcal{A}_i$ is an attention vector $\mathbf{w}_i \in \mathbb{R}^{\text{deg}_i + 1}$, which captures the attention between $v_i$ and its neighbors. $\mathcal{A}_i$ aggregates features of $\mathrm{N}_i$ as
\begin{equation}\small \label{eq_aggregate}
 \mathcal{A}_i(\mathrm{N}_i) = {\mathbf{w}_i \mathrm{Concat}_r(  \{\mathbf{z}_i\} \cup \{\mathbf{z}_j:v_j\in \mathrm{N}_i\} )} / {\mathrm{deg}_i + 1}.
\end{equation}
$\mathbf{w}_i$ is initialized with minimal values and updated iteratively.

\subsubsection{Communicable Functions}
Parameters of communicable functions are exchangeable, allowing agents to broaden their collective perception through sharing. By communicating, agents can learn from others' perceptions and integrate this knowledge into their own understanding. These functions encompass:

$\blacktriangleright$ Embedding function $\mathcal{M}_i: \mathbb{R}^{1\times d_z}\to\mathbb{R}^{1\times d_L}$ that embeds node features into the label space.
We have proved that an effective embedding can be generated using a $d_z \times d_L$-dimensional trainable matrix (\emph{c.f.} Corollary~\ref{cor_1}).

$\blacktriangleright$ Degree inference function $\mathcal{D}_i: \mathbb{R}^{1\times d_z}\to\mathbb{R}^{1\times \mathrm{deg}_{\mathrm{max}}}$ that predicts the probability distribution of a node's degree across the range of 1 to $\mathrm{deg}_{\mathrm{max}}$ (i.e., one-hot encoding) using a given feature. Instantiating $\mathcal{D}_i$ as a a $d_z \times \mathrm{deg}_{\mathrm{max}}$-dimensional trainable matrix can provide sufficient capacity for fitting (\emph{c.f.} Theorem~\ref{thm_2}).

$\blacktriangleright$ Neighboring confidence function $\mathcal{N}_i: \mathbb{R}^{1\times 2d_z}\to\mathbb{R}$, a binary-classifier that infers whether two given node features are ground-truth neighbors, and the conclusion is provided in the form of a confidence value. We have proved that any two nodes on $\mathcal{G}$ can be inferred neighborhood relationship by a $2d_z \times 1$-dimensional trainable matrix (\emph{c.f.} Theorem~\ref{thm_3}).


\subsection{I/O Interface}
\subsubsection{Input}
In a single communication round, node $v_i$ receives the following messages from its neighbors $\mathrm{N}_i$.
 1) Communicable functions of all nodes in $\mathrm{N}_i$, i.e., sets of functions $\mathcal{M}_{i}^{\mathrm{rec}} = \{\mathcal{M}_j:v_j \in \mathrm{N}_i\}$, $\mathcal{D}_{i}^{\mathrm{rec}} = \{\mathcal{D}_j:v_j \in \mathrm{N}_i\}$, and $\mathcal{N}_{i}^{\mathrm{rec}} = \{\mathcal{N}_j:v_j \in \mathrm{N}_i\}$.
2) Features from the 1-hop neighbors $Z^{\mathrm{I}}_i=\{\mathbf{z}_j: v_j\in \mathrm{N}_i\}$.
3) Sampled features from the 2-hop neighbors $Z^{\mathrm{II}}_i=\bigcup_{v_j \in \mathrm{N}_i} \mathrm{S}(v_j) - Z^{\mathrm{I}}_i$, where $\mathrm{S}(\cdot)$ denotes the sample function.

\subsubsection{Output}
Node $v_i$ sends the following messages to its neighbors $\mathrm{N}_i$.
 1) Feature in the current communication round $\mathbf{z}_i$.
2) Sampled features from node $v_i$'s neighbor. In this approach, $v_i$ samples (i.e., operates $\mathrm{S}(\cdot)$) from its received immediate neighbor features $Z^{\mathrm{I}}$ and sends them out. The sampling function of $v_i$ aims to sample the neighbors' features that have a relatively large difference from its own feature. That is, for a specific receiving nodes $v_k$, this approach reduces the similarity between the $Z^{\mathrm{I}}_k$ and $Z^{\mathrm{II}}_k$, while reducing the amount of messages transmitted outward. Therefore, given the sampling size $\rho$, $\mathrm{S}(\cdot)$ is defined as
\begin{equation}\small\label{eq_sample}
\mathrm{S}(v_i) = \left\{ \begin{array}{l}
                                    \mathrm{N}_i,  \rho \leq |\mathrm{N}_i| \\
                                    \{ \mathbf{z}_j : v_j \in {N}_i, \mathrm{rank}(||\mathbf{z}_j, \mathbf{z}_i ||_2) \leq \rho \} , \rho > |\mathrm{N}_i|
                                  \end{array}\right.,
\end{equation}
where $\mathrm{rank}(\cdot)$ denotes the rank, a smaller rank indicating a larger value.
3) Communicable functions $\mathcal{M}_{i}$, $\mathcal{D}_{i}$, and $\mathcal{N}_{i}$.

\subsection{Decision Rules}


The decision rules define how to choose actions in different states~\cite{farmer2009economy}. In GAgN, they guide the training of function parameters in the action space, based on $v_i$'s states after receiving neighbors' messages. These functions adjust parameters via the loss function using stochastic gradient descent (SGD).

 \begin{figure}[htb]
    \centering
    \includegraphics[width=0.43\textwidth]{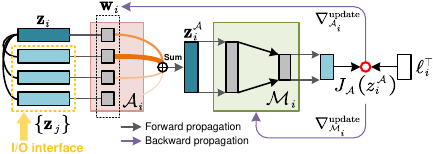}
    \caption{Independent training flow of $\mathcal{A}_i$ and $\mathcal{M}_i$.}
    \label{fig_decision1}
\end{figure}

\subsubsection{Aggregator}
The loss function $J_{\mathcal{A}}(\cdot)$ of $\mathcal{A}_i$ assigns different weights to $v_i$'s neighbors to achieve a more accurate embedding closer to its label. For the aggregated feature $\mathbf{z}^{\mathcal{A}}_i = \mathcal{A}_i(\mathrm{N}_i)$, the cross-entropy loss is:
$  J_{\mathcal{A}}(z^{\mathcal{A}}_i) =- \log \mathcal{M}_{i}(\mathbf{z}^{\mathcal{A}}_i)\mathbf{\ell}_i^{\top}.$
Parameters of $\mathcal{A}_i$ are updated via SGD-based backpropagation. Notably, while $\mathcal{M}_{i}$ participates in forward propagation, updates for $\mathcal{A}_i$ and $\mathcal{M}_{i}$ occur asynchronously to prevent the leakage of agents' private understandings (i.e., $\mathcal{A}_i$'s parameters) to their neighbors before secure communication through end-to-end training. Consequently, the gradient for updating $\mathcal{A}_i$ is
$\nabla_{\mathcal{A}_i}^{\mathrm{update}} = \nabla_{\mathcal{A}_i}\left(J_{\mathcal{A}}(\mathbf{z}^{\mathcal{A}}_i)\right).$
To obtain $\nabla_{\mathcal{A}_i}^{\mathrm{update}}$, we perform iterative training while freezing the parameters of $\mathcal{M}_{i}$.

\subsubsection{Embedding Func.}

 During training, parameters of $\mathcal{A}_{i}$ are frozen before updating the parameters of $\mathcal{M}_{i}$, i.e.,
$
  \nabla_{\mathcal{M}_i}^{\mathrm{update}} = \nabla_{\mathcal{M}_i}\left(J_{\mathcal{A}}(\mathbf{z}^{\mathcal{A}}_i)\right).
$
$\mathcal{A}_i$ and $\mathcal{M}_i$ share the same forward propagation process. Upon explicit request for specific tasks, i.e., when the training of the corresponding model parameters is initiated, $\mathcal{A}_i$ and $\mathcal{M}_i$ independently perform back propagation based on $J_{\mathcal{A}}(z^{\mathcal{A}}_i)$, thus maintaining the non-communicability of $\mathcal{A}_i$. This is illustrated in Figure~\ref{fig_decision1}.

\subsubsection{Inference Func.}

Agents use the I/O interface to access neighboring features and degrees. $\mathcal{D}_i$ learns to associate these features with degrees. Neighbors' true degrees, provided as one-hot vectors, guide $\mathcal{D}_i$'s updates, mapping features to discrete degree distributions. The attention vector $W_i$ adjusts the loss weight for each neighbor, reducing the impact of those with lower attention weights.

We denote the matrix representation of $\mathcal{D}_i$'s inference as
\begin{equation}\small \label{eq_XD}
\mathcal{X}_i^{\mathcal{D}} = \sigma(\mathrm{Concat}_r(\mathcal{D}_i(\{\mathbf{z}_j:v_j\in \mathrm{N}_i\} \cup \mathbf{z}_i))),
\end{equation}
where $\sigma$ is the nonlinear activation function, and the concatenated one-hot vector that encodes $v_i$'s and its neighbors' degrees as the supervisor:
 \begin{equation}\small \label{eq_YD}
\mathcal{Y}^{\mathcal{D}}_i = \mathrm{OneHot}(\{ d_j:v_j\in\mathrm{N}_i \} \cup \mathrm{deg}_i).
\end{equation}
Then, the loss function of $\mathcal{D}_i$ is defined as a Kullback-Leibler divergence (denoted as $\mathrm{KL}(\cdot)$) weighted by neighbor attention, i.e.,
$
  J_{\mathcal{D}}(\mathcal{X}_i^{\mathcal{D}}) = - \frac{\mathrm{KL}(\mathcal{X}_i^{\mathcal{D}} \| \mathcal{Y}^{\mathcal{D}}_i) \cdot \sum_{j} \mathbf{I}_{j} \mathbf{w}_{i,j}}{\mathrm{deg}_i+1},
$
where $\mathbf{I}$ is an all-ones vector with $\mathrm{deg}_{\mathrm{max}}$ elements, used to compute the sum of the elements in each row of a matrix.

\subsubsection{Neighboring Confidence Func.}

 $v_i$ uses $\mathcal{N}_i$ to identify neighboring relationships based on node features. The training data for $\mathcal{N}_i$ includes features within $v_i$'s view: $\mathbf{z}_i$ and $\{ \mathbf{z}_j : v_j \in \mathrm{N}_i \}$. In a graph, distinguishable decision regions exist for similarity patterns between $v_i$ and its first- and second-order neighbors~\cite{perozzi2014deepwalk}. Thus, $\mathcal{N}_i$ uses samples from the current communication round to train on two categories: 1) feature pairs labeled as "neighbor"
\begin{equation}\small \label{eq_training_samples_neighbor}
  \mathcal{X}^{\mathcal{N}}_1 = \{\mathrm{Concat}_c(\mathbf{z}_i,\mathbf{z}_j), \mathrm{Concat}_c(\mathbf{z}_j,\mathbf{z}_i):v_j\in Z^{\mathrm{I}}_i\},
\end{equation}
and 2) the feature pairs that are labeled as ``non-neighbor'' with training samples
\begin{equation}\small \label{eq_training_samples_non-neighbor}
  \mathcal{X}^{\mathcal{N}}_0 = \{\mathrm{Concat}_c(\mathbf{z}_i,\mathbf{z}_k), \mathrm{Concat}_c(\mathbf{z}_k,\mathbf{z}_i):v_k\in Z^{\mathrm{II}}_i\}.
\end{equation}
Then, according to the training samples $\mathcal{X}^{\mathcal{N}} = \mathcal{X}^{\mathcal{N}}_1 \cup \mathcal{X}^{\mathcal{N}}_0$, the loss function of $\mathcal{D}_i$ is defined as the binary cross-entropy loss (denoted as $\mathrm{CE}(\cdot)$), i.e.,
$
J_{\mathcal{N}}(\mathcal{X}^{\mathcal{N}}) = \mathrm{CE}(\mathcal{X}^{\mathcal{N}}_1,\mathcal{X}^{\mathcal{N}}_0)
$


All loss functions are differentiable, allowing functions to self-train based on their losses. Using SGD for parameter adjustment, these functions reduce the loss and improve inference capabilities.

\subsection{Action Update Rules}

%

Action update rules are essential for updating agents' actions (i.e., functions) based on functions received from their neighbors.
Using these rules, agents can dynamically adjust their actions to adapt to the evolving environment and enhance collective performance.

Node $v_i$ acquires the integrated feature of its neighbors' actions through weighted fusion of the transmitted actions in the current round. Based on these synthesized features, $v_i$ updates its actions accordingly. The weighted fusion methods are detailed in the following subsections.
\subsubsection{Embedding Func.}
These functions affects classification results, which adversaries may target, and any node in $\mathcal{V}$ could be compromised (some neighbors of $v_i$ might be malicious). Thus, $v_i$ cannot fully trust its neighbors' $\mathcal{M}_{\cdot}$. Since attention weights are individual functions and resist malicious message accumulation, the weighted fusion of $\mathcal{M}_{i}^{\mathrm{rec}}$ using $\mathbf{w}_i$ can filter harmful gradients, enabling a safe update of $\mathcal{M}_i$. Specifically, the parameter of $\mathcal{M}_i$ is updated as
\begin{equation}\small \label{eq_update_M}
\bm{\theta}_{\mathcal{M}_i} \gets \bm{\theta}_{\mathcal{M}_i} + \eta_{\mathcal{M}} \omega_{i,j}\sum_{\mathcal{M}_j\in \mathcal{M}_{i}^{\mathrm{rec}} } \left(\bm{\theta}_{\mathcal{M}_j} - \bm{\theta}_{\mathcal{M}_i}\right)  / |\mathrm{N}_i|,
\end{equation}
where $\omega_{i,j}$ is the $j$-th attention in $\mathbf{w_i}$ and $\eta_{\mathcal{M}}$ is the learning rate for updates that is usually the same as that for $\mathcal{M}_i$ itself.

\subsubsection{Degree Inference Func.}

For identical features, different degree inference functions may yield varying results. To harmonize these, $v_i$ trains a middleware function $\mathscr{D}_i$ to ensure that inference results for $\mathbf{z}_k, \forall v_k \in \mathcal{V}$ by $\mathscr{D}_i$ closely match those from the received set $\mathcal{D}_{i}^{\mathrm{rec}}$. This involves initializing $\mathscr{D}_i$ randomly and designing a loss function based on spatial random sampling for training, aiming to integrate stored functions and update $\mathcal{D}_i$. The training involves:
   1) Implementing minor random perturbations in Euclidean space to sample features $\mathbf{z}_s$ akin to $\mathbf{z}_i$, calculated as $\mathbf{z}_s = \mathbf{z}_i + \xi \mathbf{x}$, where $\mathbf{x}$ is an i.i.d. 0-1 Gaussian vector and $\xi$ the sample range.
   2) Using the features of nodes in $Z^{\mathrm{II}}_i$ as negative samples to broaden the sample diversity.
   3) Applying $\mathscr{D}_i$ and $\mathcal{D}_{i}^{\mathrm{rec}}$ to separately infer these samples and utilizing the mean squared error (MSE) to evaluate the differences in inference outcomes.
Training strategies for $\mathscr{D}_i$ are depicted in Figure~\ref{fig_Update2}.

In summary, the loss function of $\mathscr{D}_i$ is
{\small
\begin{multline}\label{eq_update_D}
\hspace{-1em}
  J_{\mathscr{D}}(v_i) = \sum_{\mathcal{D}_{j} \in \mathcal{D}_{i}^{\mathrm{rec}}} \omega_{i,j}  \Big[ - \sum_Q{\mathbb{E}}_{\mathbf{z}_s\sim P\left( \mathbf{z}_i \right)} {\mathrm{MSE}\left( \mathscr{D}_i( \mathbf{z}_s ) ,\mathcal{D}_j( \mathbf{z}_s ) \right)} \\ -  \sum\nolimits_{v_k \in Z^{\mathrm{II}}_i}{\mathrm{MSE}\left( \mathscr{D}_i( \mathbf{z}_k ), \mathcal{D}_j( \mathbf{z}_k ) \right)}\Big],
\end{multline}}
where $P$ is a sampling distribution in the Euclidean space, $Q$ defines the number of sample. After training that employing $J_{\mathscr{D}}(v_i)$ as the loss, the parameter of $\mathcal{D}_i$ is updated as
$
\bm{\theta}_{\mathcal{D}_i} \gets  \bm{\theta}_{\mathcal{D}_i} + \eta_{\mathcal{D}} \left( \bm{\theta}_{\mathcal{D}_i} - \bm{\theta}_{\mathscr{D}_i} \right).
$

 \begin{figure}[htb]
    \centering
    \includegraphics[width=0.48\textwidth]{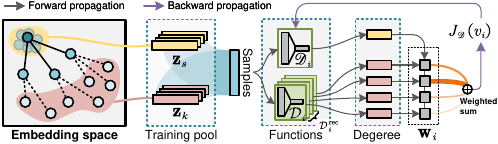}
    \caption{The training strategy for middleware function $\mathscr{D}_i$.}
    \label{fig_Update2}
\end{figure}

\subsubsection{Neighboring Confidence Func.}

From Eqs.~\eqref{eq_training_samples_neighbor} and~\eqref{eq_training_samples_non-neighbor}, under a limited view, the training set always includes $z_i$ when training $\mathcal{N}_i$ for any node $v_i$. This means $\mathcal{N}_i$ can only detect neighboring relationships for itself and a node with a given feature. The action update mechanism allows for generalization. By receiving neighboring functions $\mathcal{N}_{i}^{\mathrm{rec}}$ via the I/O interface, $\mathcal{N}_i$ can expand its scope to a richer dataset, learning from neighbors' inferential capabilities.
1) Construct a training set as per Eqs.~\eqref{eq_training_samples_neighbor} and~\eqref{eq_training_samples_non-neighbor}.
2) Use both $\mathcal{N}_i$ and $\mathcal{N}_{i}^{\mathrm{rec}}$ to predict outputs of the same samples.
3) For a neighbor $v_k$, $\mathcal{N}_k$'s output measures the similarity in inferential ability with $\mathcal{N}_i$. Similarity implies alignment; otherwise, a higher loss prompts $\mathcal{N}_i$ to adjust towards $\mathcal{N}_k$.
4) Iterative training with all received $\mathcal{N}_{i}^{\mathrm{rec}}$ completes function fusion.

The evaluation of $\mathcal{N}_k$ for $\mathcal{N}_i$ is quantified by the absolute differences between their outputs. After all functions in $\mathcal{D}_{i}^{\mathrm{rec}}$ provide evaluations, the loss is calculated by a weighted sum. This loss is then used to compute the gradient updates for a single update round of $\mathcal{N}_i$:
\begin{equation}\label{eq_update_N}\small
  \bm{\theta}_{\mathcal{N}_i} \gets \bm{\theta}_{\mathcal{N}_i} - \eta_{\mathcal{N}} \frac{\displaystyle \partial \left( \sum_{x_k \in \mathcal{X}^{\mathcal{N}}} \sum_{\mathcal{N}_{j} \in \mathcal{N}_{i}^{\mathrm{rec}}} \omega_{i,j} \left| \mathcal{N}_i(x_k) - \mathcal{N}_j(x_k) \right| \right)} { \partial \bm{\theta}_{\mathcal{N}_i} \mathrm{deg}_i |\mathcal{X}^{\mathcal{N}}_0 \cup \mathcal{X}^{\mathcal{N}}_1| }.
\end{equation}
Through multiple rounds of training, the collective guidance of $\mathcal{N}_j, \forall v_j \in \mathrm{N}_i$ on $\mathcal{N}_i$'s gradient descent direction is achieved through continuous evaluations, progressively fusing their inferential abilities into $\mathcal{N}_i$.

\subsection{Filtering Adversarial Edges}
As communication between agents tends towards convergence, they initiate the detection of adversarial edges. Given that all agents possess the capacity for inference, an intuitive approach would be to allow inter-agent detection until a sufficient and effective number of adversarial edges are identified. However, this method necessitates the execution of $|\mathcal{V}|\times|\mathcal{V}|$ inference calculations, the majority of which are redundant and do not proportionately contribute to the global detection rate relative to the computational power consumed. Therefore, we propose a multilevel filtering method that significantly reduces computational load while preserving a high detection rate.

Specifically, the entire detection process is grounded in the functions trained from the agents' first-person perspectives, with each agent acting both as a detector and a detectee. The detection procedure is illustrated by taking node $v_i$ as the detector and node $v_x$ as the detectee:
\begin{itemize}[left=0.05cm]
  \item[1.] For any node $v_x$, if there is a neighbor with insufficient attention, the edge formed by this neighbor relationship is considered as a suspicious edge.
  \item[2.] Random proxy communication channels are established, introducing an unfamiliar neighbor $v_i$ to $v_x$, with both parties exchanging information through I/O interfaces.
  \item[3.] $v_i$ infers the degree of $v_x$ based on its own experience (i.e., applies $\mathcal{D}_i(\mathbf{z}_x)$); if the inferred result deviates from the actual outcome, the process advances to the subsequent detection step; otherwise, a new detectee is selected.
  \item[4.] $v_i$ estimates neighbor confidence between $v_x$ and its neighbors using its experience (i.e., computing $\mathrm{Concat}_c(\mathbf{z}_x,\mathbf{z}_u)$, $\forall v_u \in \mathrm{N}_x$); edges with trust levels below a threshold are identified as adversarial edges.
\end{itemize}

Employing this approach, administrators possessing global model access rights are limited to establishing proxy connections, which entails creating new I/O interfaces for nodes without the ability to interfere in specific message passing interactions between them. As a result, distributed agents are responsible for conducting detection tasks, enabling individual intelligences to identify adversarial perturbations derived from global training.

\subsection{Theoretical Analysis}
We first establish the global equivalence.
\begin{theorem}[Equivalence]\label{thm_1}
  Denote $\mathbb{M}(\mathcal{G})=\{\mathbf{z}^{\mathbb{M}}_i = \mathcal{M}_i(\mathbf{z}_i):v_i \in \mathcal{V}\}$ as the embedding result of all nodes, $\texttt{SAGE}(\mathcal{G})=\{ \mathbf{z}^{\texttt{S}}_i,...\mathbf{z}^{\texttt{S}}_{|\mathcal{V}|} \}$ as the classification result of GraphSAGE, and $G$ as the set of all attributed graphs. For any $\texttt{SAGE}$ there exists $\mathbb{M}$ that
  \begin{equation}\label{eq_Equivalence}
    \mathbb{M}(\mathcal{G}) = \texttt{SAGE}(\mathcal{G}), \forall \mathcal{G} \in G.
  \end{equation}
\end{theorem}

We then propose the corollary to guide the design of the embedding function:

\begin{corollary}\label{cor_1}
If model $\mathcal{M}_i$ is a single-layer neural network with trainable parameters as a $d_z \times d_L$ matrix, then $\mathbb{M}_i$ is universally capable of embedding any feature space onto the corresponding label space.
\end{corollary}

For the structure of the degree inference function, we arrive at the following conclusions:

\begin{theorem}\label{thm_2}
  Given an attribute graph $\forall \mathcal{G} \in G$, it is possible to map its feature matrix to any label matrix using only a single linear transformation (i.e., a trainable matrix) and a nonlinear activation.
\end{theorem}

For the neighboring confidence function's structure, we present the following findings:

\begin{theorem}\label{thm_3}
For any graph $\mathcal{G}$, there exists a fixed $2d_z$-dimensional weight vector $\mathbf{q}=[q_1, \ldots, q_{2d_z}]$, such that for any two nodes $v_i$ and $v_j$ in $\mathcal{G}$ and their true neighboring relationship $\phi$ (where $\phi$ is a relative minimum or converges to 1), $\mathrm{Concat}_r(\mathbf{z}_i,\mathbf{z}_j)\mathbf{q}^{\top} = \phi$ is solvable.
\end{theorem}

\begin{table*}[htbp]
  \centering \small
\setlength{\tabcolsep}{1.25mm}{
    \begin{tabular}{c | c | ccccc c | ccccc c}
    \toprule
        \multicolumn{2}{c |}{Surrogate} & \multicolumn{6}{c |}{GCN (primary attack by Metattack)}  & \multicolumn{6}{c }{Corresponding defenses (secondary attack by G-EPA)}  \\
    \cmidrule(lr){1-2}\cmidrule(lr){3-8} \cmidrule(lr){9-14}
        Dataset  & $r_p$   & RGCN & SVD & Pro & Jaccard & EGNN & GAgN & RGCN & SVD & Pro & Jaccard & EGNN & GAgN \\
  \cmidrule(lr){1-2}\cmidrule(lr){3-8} \cmidrule(lr){9-14}
    Cora & 20  & 58.67 & 57.01 & 63.94 & \underline{72.51} & 69.02 & \textbf{77.01} & 49.17\scriptsize{±0.89} & 47.91\scriptsize{±2.18} & 55.04\scriptsize{±1.11} & \underline{63.11\scriptsize{±1.84}} & 59.52\scriptsize{±1.39} & \textbf{74.91\scriptsize{±0.74}} \\
    Citeseer & 20  & 62.53 & 57.29 & 56.24 & \underline{66.21} & 64.94 & \textbf{70.49} & 54.03\scriptsize{±2.00} & 48.79\scriptsize{±1.77} & 47.74\scriptsize{±1.59} & \underline{57.71\scriptsize{±1.57}} & 56.44\scriptsize{±1.94} & \textbf{69.19\scriptsize{±0.91}} \\
    Polblogs & 20  & 58.36 & 54.87 & 73.10 & 69.87 & \underline{75.42} & \textbf{80.92} & 49.86\scriptsize{±1.17} & 46.37\scriptsize{±2.71} & 64.60\scriptsize{±1.35} & 61.37\scriptsize{±2.07} & \underline{66.92\scriptsize{±1.54}} & \textbf{78.42\scriptsize{±2.38}} \\
    Pubmed & 20  & 71.20 & 81.24 & \underline{82.82} & 76.39 & 79.06 & \textbf{82.98}  & 62.70\scriptsize{±1.34} & 72.74\scriptsize{±0.94} & \underline{74.07\scriptsize{±0.75}} & 67.89\scriptsize{±0.73} & 70.56\scriptsize{±0.79} & \textbf{78.29\scriptsize{±0.67}} \\
    \bottomrule
    \end{tabular}
    }
    \caption{Classification accuracy (\%) under primary and secondary attackm with the highest scores in \textbf{bold} and the second highest \underline{underlined}. $r_p$ is the perturbation rate. SVD and Pro is the abbreviation of GNN-SVD and Pro- GNN respectively.}
  \label{tab:GlobalAcc}%
\end{table*}%

\begin{figure*}[htb]
    \centering
    \begin{minipage}[b]{0.325\textwidth}
        \centering
        \includegraphics[width=\textwidth]{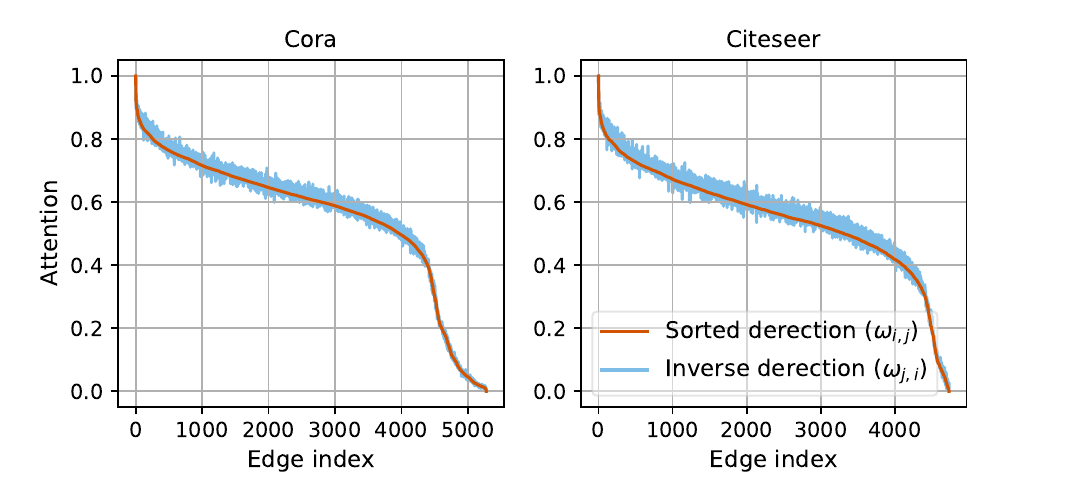}
        \caption{Attentions of random direction \& inverse direction on the same edges.}
        \label{fig_sym}
    \end{minipage}
    \hspace{0.03cm}
    \begin{minipage}[b]{0.308\textwidth}
        \centering
        \includegraphics[width=\textwidth]{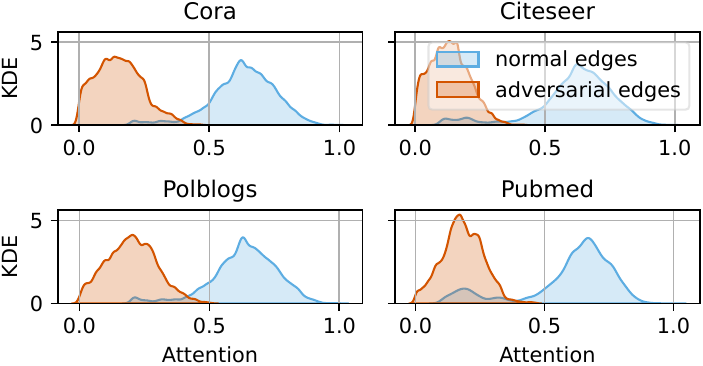}
        \caption{Distribution of attentions on normal and adversarial edges.}
        \label{fig_KDE}
    \end{minipage}
    \hspace{0.03cm}
    \begin{minipage}[b]{0.300\textwidth}
        \centering
        \includegraphics[width=\textwidth]{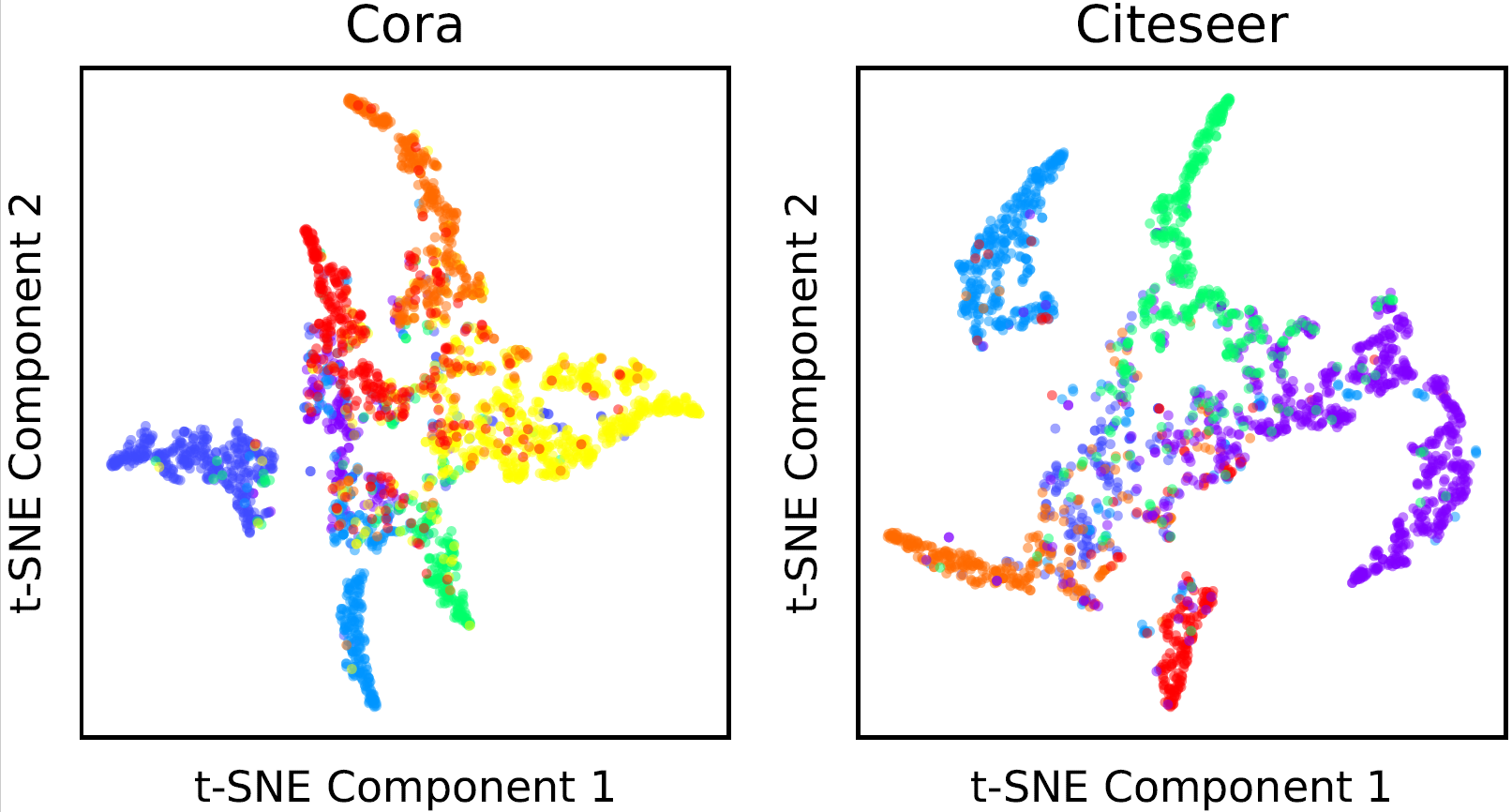}
        \caption{Well-trained agents generated embedding.}
        \label{fig_emb_single}
    \end{minipage}
\end{figure*}

\section{Experiments}


\textbf{Datasets}. Our approaches are evaluated on six real-world datasets widely used for studying graph adversarial attacks~\cite{Sun_2020_WWW,Liu_2022_TNNLS, entezari2020SVD,li2022STABLE,liu2024towards}, including {Cora}, {Citeseer}, {Polblogs}, and {Pubmed}.

\noindent \textbf{Baselines.} The GAgN model protects non-defense GNNs from edge-perturbing attacks and outperforms other defensive approaches. Baseline models are evaluated as follows:
\textit{Comparison defending models.} GAgN is compared against:
   1) {RGCN}, which uses Gaussian distributions for node representations to mitigate adversarial effects,
   2) {GNN-SVD}, which applies truncated SVD for a low-rank adjacency matrix approximation,
   3) {Pro-GNN}~\cite{jin2020Pr0GNN}, which enhances GNN robustness through intrinsic node properties,
   4) {Jaccard}~\cite{wu2019jaccard}, defending based on Jaccard similarity,
   5) {EGNN}~\cite{liu2021EGNN}, filtering perturbations via $\ell_1$- and $\ell_2$-based smoothing.
\textit{Attack methods.} Experiments incorporate:
   1) {Metattack}, a meta-learning based strategy,
   2) {G-EPA}~\cite{Liu_2022_TNNLS}, a generalized edge-perturbing attack approach.
The perturbation rate for attacks is set at 20\%, a standard in the field unless otherwise specified.

%


\subsubsection{Classification Accuracy}
We evaluate the global classification accuracy of the proposed GAgN against primary and secondary attacks. Our experimental design uses GCN and its defenses as surrogates for Metattack and G-EPA, where Metattack represents the most potent primary attack and G-EPA an effective secondary strategy. Accuracy results are detailed in Table~\ref{tab:GlobalAcc}.
As GAgN closes global input-output interfaces, it cannot directly serve as a surrogate in secondary attacks. To compensate, we transfer perturbed graphs from baseline models under secondary attacks to the GAgN test set. Five datasets are conducted, presenting average accuracies and variation ranges to illustrate GAgN’s resilience against secondary attacks. The diversity of the baselines likely covers potential vulnerabilities in defenses.

Our results indicate that GAgN achieves the best classification accuracy under both primary and secondary attacks.
This is particularly noteworthy, as the lack of global input-output interfaces in GAgN makes it difficult for adversaries to exploit its vulnerabilities through global training, rendering it nearly impervious to secondary attacks.

\subsubsection{Effectiveness of Aggregator}~\label{sec_sym}

\noindent \underline{Symmetry of Attentions}.
In a well-trained GAgN, adjacent agents should learn similar attentions for their connecting edges (i.e., $\omega_{i,j} \approx \omega_{j,i}$) under the absence of collusion. To validate this, we first arbitrarily assign directions to the edges in the graph, defining the attention of any edge $e_{i,j}$ as $\omega_{i,j}$ from node $v_i$ to $v_j$. This results in a new graph with reversed attentions for all edges. We use a line chart to compare forward and reverse attentions for the same edges after learning in the GAgN model. For a certain edge's smoothness, we sort the forward attention, obtain the sorted edge index, and use this index to find the corresponding reverse attention, examining variations in attentions.
As shown in Figure~\ref{fig_sym}, the variations between forward and reverse attentions are nearly identical, indicating that adjacent agents learn similar attentions. This confirms the effectiveness of the neighbor feature aggregator in the GAgN model.
\emph{To demonstrate consistency among adjacent agents in learning attention for the same edge, we selected the node with the highest degree (168) from Cora. Experimental results, presented in Appendix~\ref{sec_APP_sym}, show the learning curves of neighboring agents' attentions on the same edge.}



\noindent \underline{Distribution of Attentions}.
Here we investigate the attention distribution learned by agents for both normal edges and adversarial edges. A significant distinction between the two would indicate that agents, through communication, have become aware that lower weights should be assigned to adversarial edges, thereby autonomously filtering out information from illegitimate neighbors introduced by such edges. As the previous set of experiments have demonstrated the symmetry of attention, we present the attention of a randomly selected agent on one end of the edge. The experimental results, as depicted in Figure~\ref{fig_KDE}, reveal a notable difference in the kernel density estimation (KDE) of attentions between normal edges and adversarial edges (generated by Metattack). This outcome substantiates that agents, by training the aggregator, have acquired the capability to autonomously filter out adversarial edges.


\subsubsection{Effectiveness of Embedding}

After communication, the agent uses a well-trained embedding function to map its feature into the label space. We visualize the global embeddings to assess their effectiveness. If agents' embeddings, trained only on their labels under limited knowledge, still show global clustering, it indicates the embedding function's success. Figure~\ref{fig_emb_single} shows the t-SNE visualization of Cora and Citeseer embeddings~\cite{van2008visualizing}, with points of different colors indicating distinct categories. The clear clustering of points validates the embedding function's efficacy.
\emph{Additionally, the visualization of the embedding process is presented in Appendix~\ref{sec_emb_sym}, illustrating how GAgN progressively learns effective embeddings.}

\begin{table}[ht]
    \centering \small
    \begin{tabular}{c|c c c c}
    \toprule
    & Cora & Citeseer & Polblogs & Pubmed \\
    \midrule
    L-N & 90.35 & 89.21 & 92.41 & 81.08 \\
    S-N & 95.31 & 94.19 & 96.32 & 83.45 \\
    D-N & 92.64 & 91.08 & 92.57 & 83.06 \\
    \bottomrule
    \end{tabular}
    \caption{Degree inference accuracy (\%)}
    \label{tab_degree}
\end{table}

\subsubsection{Effectiveness of Degree Inference}

Each agent possesses the ability to infer the degree of any node based on the degree inference function. To validate the effectiveness of this inference, we select 150 representative nodes as test sets after training the GAgN model on the corresponding clean graph, following these rules: 1) the top 50 nodes with the highest degree, denoted as L-N; 2) the top 50 nodes with the lowest degree, denoted as S-N; 3) the 50 nodes furthest away from the inference node in terms of graph distance, denoted as D-N. Notably, to avoid training set leakage into the test set, an agent's 1-hop neighbors and select 2-hop neighbors are excluded from its test set. The experimental results, as shown in Table~\ref{tab_degree}, demonstrate that even without data on the test nodes, agents can generalize their inference capabilities to the nodes of the graph under limited knowledge training, effectively inferring the degree of nodes.


\subsubsection{Effectiveness of Neighboring Confidence}


To validate the neighboring confidence function on graph $\mathcal{G}$, $|\mathcal{V} \times \mathcal{V}|$ matrix operations are needed. To reduce computational complexity and ensure fair evaluation, we construct four representative graphs and use their neighbor relationships as test samples:
1) The original graph $\mathcal{G}_{0}$, with all labels set to 1.
2) A graph $\mathcal{G}_{1}$ with randomly distributed edges, labeled based on whether edges are original.
3) A randomly rewired graph $\mathcal{G}_{2}$, with all labels set to 0.
4) A graph $\mathcal{G}_{3}$ with only adversarial edges, all labeled 0.
We randomly select 50 nodes to compute the confidence for neighbor relationships, treating this as a binary classification problem. Classification accuracy measures the function's effectiveness. Figure~\ref{fig_neighbora} shows that the neighboring confidence function performs well across all test sets, demonstrating its effectiveness.

\begin{figure}[htb]
\centering
\includegraphics[width=0.46\textwidth]{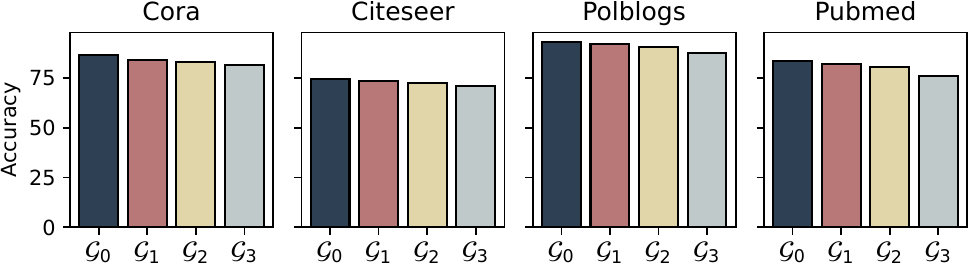}
\caption{Accuracy of neighboring Confidence function.}
\label{fig_neighbora}
\end{figure}

\begin{figure}[htb]
\centering
\includegraphics[width=0.46\textwidth]{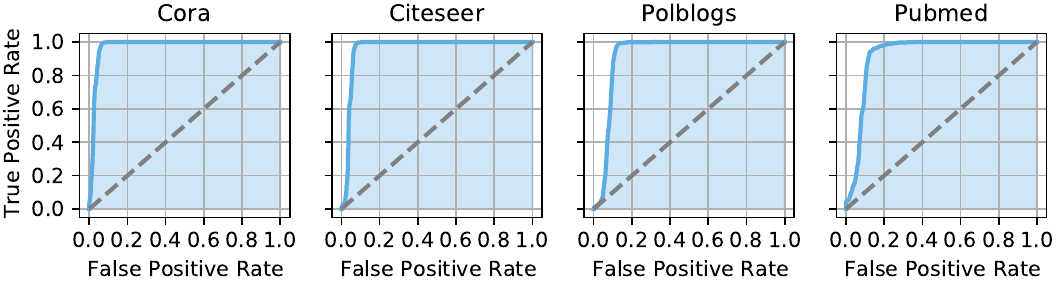}
\caption{ROC curves with different attention threshold.}
\label{fig_ROC}
\end{figure}

\subsubsection{Ablation Experiments}
We employed a multi-module filtering approach to reduce computational complexity when filtering adversarial edges. To investigate the individual effects of these modules when operating independently, it is crucial to conduct ablation experiments.
Section~\ref{sec_sym} has explored the distribution patterns of attention scores for normal and adversarial edges. Edges with low attention scores are initially considered suspicious, which implies that the attention threshold may influence the detection rate and false alarm rate. To quantitatively examine the impact of threshold selection on these rates, we utilize receiver operating characteristic (ROC) curves.
As depicted in Figure~\ref{fig_ROC}, the detection rate is relatively insensitive to the attention threshold. This allows for a high detection rate while maintaining low false alarm rates.
\begin{table}[htbp]
\small
  \centering
  \setlength{\tabcolsep}{0.8mm}
    \begin{tabular}{c|cc|cc}
    \toprule
    Dataset & \multicolumn{2}{c|}{Cora} & \multicolumn{2}{c}{Citeseer}  \\
    \midrule
    Metrics & TP \scriptsize{(TPR)} & FP \scriptsize{(FPR)} & TP \scriptsize{(TPR)} & FP \scriptsize{(FPR)} \\
        \cmidrule(lr){1-1}\cmidrule(lr){2-3} \cmidrule(lr){4-5}

    \emph{{Att.}} & 928 \scriptsize{(85.52\%)} & 302 \scriptsize{(5.56\%)} & 782 \scriptsize{(82.66\%)} & 341 \scriptsize{(7.20\%)} \\
    {\emph{Att.}+$\mathcal{D}$+$\mathcal{N}$} & 901 \scriptsize{(83.04\%)} & 47 \scriptsize{(0.87\%)} & 753 \scriptsize{(79.60\%)} & 55 \scriptsize{(1.16\%)} \\
    \bottomrule
    \end{tabular}
    \caption{Results of ablation experiments}
  \label{tab_degree}%
\end{table}%

%

After identifying suspicious edges based on attention, the degree inference and neighboring confidence functions help filter out normal edges, reducing false positives. We examine the effectiveness of these functions in detecting adversarial edges while retaining normal ones. True positives (TP) and true positive rate (TPR) measure detection accuracy, while false positives (FP) and false positive rate (FPR) measure normal edge misclassification. We use Metattack for testing. Table 3 shows results: ``\emph{Att.}'' refers to detection by attention alone, and ``\emph{Att.}+$\mathcal{D}$+$\mathcal{N}$'' refers to combined detection. Detection using only attention has a higher rate but also more false positives. Adding degree inference and neighboring confidence reduces false positives significantly without majorly affecting detection rates.

\section{Conclusions}

We proposed GAgN for addressing the inherent vulnerabilities of GNNs to edge-perturbing attacks.
By adopting a decentralized interaction mechanism, GAgN facilitates the filtration of adversarial edges and thwarts global attacks.
The theoretical sufficiency for GAgN further simplifies the model, while experimental results validate its effectiveness in resisting edge-perturbing attacks.

\section{Acknowledgments}
This work is supported in part by the National Key R\&D Program of China (No. 2022YFB3102100, 2020YFB1805400), 
the National Science Foundation of China (No. U22B2027, 62032002, 62372313, 62172297, 62476107),
the China Postdoctoral Science Foundation (No. 2024M752211),
the Sichuan Science and Technology Program (No. 24QYCX0417),
and the Youth Science Foundation of Sichuan (No. 25QNJJ4560).

\bibliography{sample-base}

\clearpage

\appendix

\section{Proofs}

\subsection{Proof of Theorem~\ref{thm_1}}~\label{sec_proof_T1}

  First, we demonstrate that the first-order message passing of $\mathbb{M}$ and $\texttt{SAGE}$ is equivalent (meaning that the depth of $\texttt{SAGE}$ is 1). This equivalence arises from the diverse calculation methods for $\texttt{SAGE}$, which depend on the selection of its aggregator. By examining different instantiations of $\mathcal{A}$ and $\mathcal{M}$, we establish their respective equivalences.

  GraphSAGE [1] that employs mean aggregator updates the feature (i.e., a single round of forward propagation) as
  \begin{equation}\label{eq_mean_agg}
    \mathbf{z}_i^{\texttt{S}} = \sigma\left( \mathbf{W}^1_{\texttt{S}} \cdot \mathrm{Mean}(\{\mathbf{z}_i\}\cup \{{\mathbf{z}_j:v_j\in \mathrm{N}_i}\})  \right),
  \end{equation}
  where $\mathbf{W}^1_{\texttt{S}}$ is the weight matrix of 1-hop neighborhood, and $\mathrm{MEAN}(\cdot)$ is the mean operator. It is not difficult to see that by instantiating $\mathbf{w}_i$ as an all-ones vector $\mathbf{I}'$ with $\mathrm{deg}_i + 1$ items, $\mathcal{A}_i$ can thus aggregate features as
  \begin{align}\label{eq_A_int_1}
    \mathcal{A}_i(\mathrm{N}_i) &= \frac{\mathbf{I}' \mathrm{Concat}_r( \{\mathbf{z}_i\} \cup \mathrm{N}_i )}{\mathrm{deg}_i + 1} \\
    &= \frac{\mathrm{Sum}( \{\mathbf{z}_i\}\cup \{{\mathbf{z}_j:v_j\in \mathrm{N}_i} \})} {\mathrm{deg}_i + 1} \notag \\
    &= \mathrm{Mean}(\{\mathbf{z}_i\}\cup \{{\mathbf{z}_j:v_j\in \mathrm{N}_i}\}).
  \end{align}
  Then, by instantiating $\mathcal{M}_i$ as a function that computes the output by left-multiplying the input with a trainable matrix $\mathbf{W}_{\mathbb{M}}$ and the nonlinear activation (i.e., instantiated as a fully-connected neural network), we have
  \begin{equation}\label{eq_M_int_1}
    \mathbf{z}^{\mathbb{M}}_i = \mathcal{M}_{i}(\mathcal{A}_i(\mathrm{N}_i)) = \sigma\left(\mathbf{W}_{\mathbb{M}} \cdot \mathrm{Mean}(\{\mathbf{z}_i\}\cup \{\mathbf{z}_j:v_j\in \mathrm{N}_i\}) \right).
  \end{equation}
  Eqs.~\eqref{eq_mean_agg} and~\eqref{eq_M_int_1} respectively represent the computation processes of 1-depth mean-aggregator $\texttt{SAGE}$ and $\mathbb{M}$, when their trainable matrix have same shape, i.e., $ \mathbf{W}_{\mathbb{M}} \sim \mathbf{W}^1_{\texttt{S}}$, their equivalence can be apparently observed. \emph{In our discussion concerning the embedding results at the agent level, it is worth noting that we will no longer employ subscript $i$ to differentiate trainable matrices $\mathbf{W}_{\mathbb{M}}$. It is crucial to recognize that $\mathbf{W}_{\mathbb{M}}$ does not represent a global model, and this notation is solely utilized to demonstrate the aforementioned theory.}

 GraphSAGE that employed mean-pooling aggregator updates the feature as
 \begin{equation}\label{eq_pool_agg}\small
  \mathbf{z}_i^{\texttt{S}} = \sigma \left(\mathbf{W}^1_{\texttt{S}} \cdot \mathrm{MeanPool}\left(\{ \sigma\left( \mathbf{z}_i,  \mathbf{z}_j \mathbf{W}_{\texttt{P}} + \mathbf{b} \right),\forall v_j \in \mathrm{N}_i \} \right)\right),
 \end{equation}
  where $\mathbf{W}_{\texttt{P}}$ is a fully-connected neural network in pooling aggregator, $\mathbf{b}$ is the bias vector.
  Next, we instantiate $\omega_{i,j}$ as

  \begin{equation}\label{eq_W_int_2}
     \omega_{i,j} = \left\{ \begin{array}{ l } 1 + \mathbf{b} \mathbf{W}_{\mathbb{M}} ^{-1} \mathbf{z}_j^{+}, i\neq j \\
     \mathbf{z}_i \mathbf{W}_{\mathbb{M}} ^{-1} \mathbf{z}_i^{+}, i = j  \end{array}   \right. ,
  \end{equation}
  where $\mathbf{z}_j^{+}\in\mathbb{R}^{d_z}$ is the pseudoinverse of $\mathbf{z}_j$.
  Then, we instantiate $\mathcal{M}_i$ as a 2-layer neural network (with trainable matrix $\mathbf{W}_{\mathbb{M}}$ and $\mathbf{W}'_{\mathbb{M}}$) and a pooling layer whose forward propagation is formulated as
  \begin{equation}\label{eq_M_int_2}
    \mathcal{M}(\mathbf{x}) = \sigma \left(\mathbf{W}'_{\mathbb{M}} \cdot   \frac{  \sigma \left(\mathbf{x}  \mathbf{W}_{\mathbb{M}}  \right) }{|\mathrm{N}_i+1|} \right).
  \end{equation}
  $\mathcal{M}_i$ can therefore embed features as
  \begin{align}\label{eq_A_int_2}
  \mathbf{z}^{\mathbb{M}}_i &=
  \mathcal{M}_{i}(\mathcal{A}_i(\mathrm{N}_i)) \notag  \\
  &= \sigma \left(\mathbf{W}'_{\mathbb{M}} \cdot   \frac{ \sigma\left(  \left( \omega_{i,i}\mathbf{z}_i + \sum_{v_j\in\mathrm{N}_i} \omega_{i,j} \mathbf{z}_j\right) \mathbf{W}_{\mathbb{M}} \right) }{|\mathrm{N}_i+1|} \right).
  \end{align}
  Since
  {\small
  \begin{align}\label{eq_pool_1}
   ( \omega_{i,i}\mathbf{z}_i &+ \sum_{v_j\in\mathrm{N}_i} \omega_{i,j} \mathbf{z}_j)\mathbf{W}_{\mathbb{M}} = \omega_{i,i} \mathbf{z}_i \mathbf{W}_{\mathbb{M}} + \sum_{v_j\in\mathrm{N}_i}  \omega_{i,j} \mathbf{z}_j \mathbf{W}_{\mathbb{M}} \notag \\
 &= \mathbf{z}_i \mathbf{W}_{\mathbb{M}} ^{-1} \mathbf{z}_i^{+}  \mathbf{z}_i  \mathbf{W}_{\mathbb{M}} + \sum_{v_j\in\mathrm{N}_i} \left(1 + \mathbf{b} \mathbf{W}_{\mathbb{M}} ^{-1} \mathbf{z}_j^{+}\right) \mathbf{z}_j \mathbf{W}_{\mathbb{M}} \notag  \\
 &= \mathbf{z}_i + \sum_{v_j\in\mathrm{N}_i} \left(\mathbf{z}_j \mathbf{W}_{\mathbb{M}} + \mathbf{b} \mathbf{W}_{\mathbb{M}} ^{-1} \mathbf{z}_j^{+} \mathbf{z}_j \mathbf{W}_{\mathbb{M}}  \right) \notag \\
 &=  \mathbf{z}_i + \sum_{v_j\in\mathrm{N}_i} \mathbf{z}_j {\mathbf{W}_{\mathbb{M}}} + \mathbf{b},
  \end{align}
  }
  substituting Eq.~\eqref{eq_pool_1} into Eq.~\eqref{eq_A_int_2} yields:
  {\small
    \begin{align}\small\label{eq_M_int_3}
  \mathbf{z}^{\mathbb{M}}_i &=
  \mathcal{M}_{i}(\mathcal{A}_i(\mathrm{N}_i)) \notag  \\
  &= \sigma \left(\mathbf{W}'_{\mathbb{M}} \cdot  \frac{ \sigma\left( \mathbf{z}_i + \sum_{v_j\in\mathrm{N}_i} {\mathbf{z}_j \mathbf{W}_{\mathbb{M}}} + \mathbf{b}\right)}{|\mathrm{N}_i+1|} \right) \notag  \\
  &= \sigma \left(\mathbf{W}'_{\mathbb{M}} \cdot \mathrm{MeanPool}\left(  \sigma \left( \{\mathbf{z}_i,  \mathbf{z}_j \mathbf{W}_{\mathbb{M}} + \mathbf{b}\}:v_j \in \mathrm{N}_i \right)  \right)\right).
  \end{align}
  }
  Eqs.~\eqref{eq_M_int_3} and~\eqref{eq_pool_agg} respectively represent the computation processes of $\mathbb{M}$ and 1-depth pool-aggregator $\texttt{SAGE}$. $\mathcal{M}_i$ enables the shape of its trainable matrices to be identical to that of $\mathbf{W}^1_{\texttt{S}}$ and $\mathbf{W}_{\texttt{P}}$, i.e., $ \mathbf{W}'_{\mathbb{M}} \sim \mathbf{W}^1_{\texttt{S}}$ and $\mathbf{W}_{\mathbb{M}} \sim \mathbf{W}_{\texttt{P}}$, leading to the equivalence between the two entities, as their computational patterns align.

The aforementioned analysis demonstrates that GAgN and GraphSAGE are equivalent in the context of 1-depth aggregation. Building upon this conclusion, we will now establish the equivalence of GAgN and GraphSAGE in the case of $K$-depth aggregation as well. We denote a single aggregation in GraphSAGE as $\mathrm{AGG}_1(\cdot)$, then the 1-depth aggregation is formulated as
\begin{equation}\label{eq_1_AGG}
  \mathbf{z}_i^{\texttt{S},K=1}  = \sigma\left( \mathbf{W}^1_{\texttt{S}} \mathrm{A{\scriptstyle GG}}_1(\mathrm{N}_i) \right).
\end{equation}
The 2-depth aggregation is
{\small
\begin{align}\label{eq_2_AGG}
  \mathbf{z}_i^{\texttt{S},K=2} &= \sigma\left( \mathbf{W}^2_{\texttt{S}} \mathrm{A{\scriptstyle GG}}_1( \{\mathbf{z}_j^{\texttt{S},K=1}:v_j \in \mathrm{N}_i\} ) \right) \notag\\
  &= \sigma\left( \mathbf{W}^1_{\texttt{S}} \mathrm{A{\scriptstyle GG}}_1( \left\{\sigma\left( \mathbf{W}^2_{\texttt{S}} \mathrm{A{\scriptstyle GG}}_1(\mathrm{N}_i) \right)\right\}:v_j \in \mathrm{N}_i\} ) \right).
\end{align}
}
We write Eq.~\eqref{eq_2_AGG} as $\mathbf{z}_i^{\texttt{S},K=2}=\mathrm{A{\scriptstyle GG}}_2(\mathrm{N}_i)$. From Eqs~\eqref{eq_1_AGG} and~\eqref{eq_2_AGG}, it can be observed that $\mathrm{A{\scriptstyle GG}}_2(\mathrm{N}_i)$ is a deep nested function [2] based on $\mathrm{A{\scriptstyle GG}}_1(\mathrm{N}_i)$. The nesting rule involves $\mathrm{A{\scriptstyle GG}}_2(\mathrm{N}_i)$ taking the output of $\mathrm{A{\scriptstyle GG}}_1(\mathrm{N}_i)$ as input, re-entering it into $\mathrm{A{\scriptstyle GG}}_1(\mathrm{N}_i)$, and replacing the parameter $\mathbf{W}^1_{\texttt{S}}$ in $\mathrm{A{\scriptstyle GG}}_1(\mathrm{N}_i)$ with $\mathbf{W}^2_{\texttt{S}}$. We denote this nested rules as

\begin{equation}\label{eq_nasted_rule}
  \mathrm{A{\scriptstyle GG}}_2(x) = \left(\mathrm{A{\scriptstyle GG}}_1^{\mathbf{W}^1_{\texttt{S}}} \circ \mathrm{A{\scriptstyle GG}}_1^{\mathbf{W}^2_{\texttt{S}}}\right)  (x),
\end{equation}
where $\mathrm{A{\scriptstyle GG}}_1^{\mathbf{W}^1_{\texttt{S}}}$ denotes the $\mathrm{A{\scriptstyle GG}}_1$ with the parameter $\mathbf{W}^1_{\texttt{S}}$.

Consequently, given an argument $x$, we can write
\begin{equation}\label{eq_nasted_AGGk}
  \mathrm{A{\scriptstyle GG}}_K(x) = \left(\mathrm{A{\scriptstyle GG}}_1^{\mathbf{W}^1_{\texttt{S}}} \circ \ldots \circ \mathrm{A{\scriptstyle GG}}_1^{\mathbf{W}^K_{\texttt{S}}}\right)  (x).
\end{equation}
It is evident that if a set of nested functions $F=\{F_1,\ldots, F_K \}$ satisfies
\begin{itemize}
  \item its initial nesting functions is equivalent to $\mathrm{A{\scriptstyle GG}}_1$, and
  \item its nesting rules are identical with Eq~\eqref{eq_nasted_AGGk},
\end{itemize}
then $F_K$ is equivalent to $\mathrm{A{\scriptstyle GG}}_K$.

During the $e$-th communication round of GAgN, node $v_j$ trains its models $\mathcal{M}_j$ and $\mathcal{A}_j$ via the data received from the I/O interface, thus learning the features of its 1-hop neighbors. We denote $\mathrm{L{\scriptstyle EARN}}^j_1(\cdot)$ as the learning function of $v_j$ within a single communication round. Then, we can write
\begin{equation}\label{eq_SGARN_e_1}
  \mathbf{z}^{\mathbb{M},e}_j  = \mathcal{M}_{j}\left(\mathcal{A}_j\left(\mathrm{N}_j\right)\right) = \mathrm{L{\scriptstyle EARN}}^j_1(\mathrm{N}_j).
\end{equation}
Subsequently, we shift our focus to any neighboring node $v_i$ of $v_j$. In the ($e$+1)-th communication round, $v_i$ learns the feature of $v_j$ by
\begin{equation}\label{eq_SGARN_e_2}
  \mathbf{z}^{\mathbb{M},e+1}_i = \mathrm{L{\scriptstyle EARN}}^i_1\left( \left\{ \mathrm{L{\scriptstyle EARN}}^j_1\left(\mathrm{N}_j\right):v_j \in \mathrm{N}_i \right\} \right).
\end{equation}
At this point, $v_j$ has already integrated the features of its 1-hop neighbors during the $e$-th communication round, and a portion of these neighbors are 2-hop neighbors of $v_i$. Consequently, after 2 communication rounds, $v_i$'s receptive field expands to include its 2-hop neighbors. The learning from 2-hop neighbors is formulated as $\mathbf{z}^{\mathbb{M},e}_j = \mathrm{L{\scriptstyle EARN}}^j_2(\mathrm{N}_j)$. From Eqs.~\eqref{eq_SGARN_e_1} and~\eqref{eq_SGARN_e_2}, it is evident that $\mathrm{L{\scriptstyle EARN}}$ is also a nested function, with the nested rule being:
\begin{equation}\label{eq_nasted_LEARN}
  \mathrm{L{\scriptstyle EARN}}^j_K(x) = \underbrace{\left(\mathrm{L{\scriptstyle EARN}}^j_1 \circ \mathrm{L{\scriptstyle EARN}}^i_1 \circ \ldots \circ \mathrm{L{\scriptstyle EARN}}_1^{j_K} \right)}_{K\text{ items}},
\end{equation}

where $v_{j_{K}}$ is the arbitrary K-hop neighbor of $v_j$. The equivalence of and $\mathbb{M}$ and 1-depth aggregation supports that $\mathrm{L{\scriptstyle EARN}}^j_1=\mathrm{A{\scriptstyle GG}}_1$. Further, since the agents enables the shape of its trainable matrices to be identical to that of $\{\mathbf{W}^1_{\texttt{S}},\ldots,\mathbf{W}^K_{\texttt{S}}\}$, we can deduce that the nested rule of the learning function of GAgN (formulated as Eq.~\eqref{eq_nasted_LEARN}) and the aggregator in GraphSAGE (formulated as Eq.~\eqref{eq_nasted_AGGk}) are equivalent. Consequently, the equivalence between $\mathbb{M}$ and $K$-depth pool-aggregator $\texttt{SAGE}$ is thus established.


\subsection{Proof of Corollary~\ref{cor_1}}~\label{sec_proof_C1}

Suppose there exists a Turing-complete function $\mathcal{T}$, capable of performing precise node-level classification. Thanks to the theory of GNN Turing completeness [3], i.e., a GNN that has enough layers is Turing complete. This theorem can be formulated as: $\mathrm{A{\scriptstyle GG}}_{\infty}=\mathcal{T}$. Based on the conclusions derived from Theorem~\ref{thm_1}, we have that the sufficiently communicated function $\mathrm{L{\scriptstyle EARN}}_{\infty}$ is equivalent to $\mathrm{A{\scriptstyle GG}}_{\infty}$. The computational accessibility of M for embedding tasks can thus be demonstrated by the transitive equivalence chain $\mathcal{M}_i=\mathrm{L{\scriptstyle EARN}}_{\infty}=\mathrm{A{\scriptstyle GG}}_{\infty}=\mathcal{T}$.

\subsection{Proof of Theorem~\ref{thm_2}}~\label{sec_proof_T2}

First, we derive the characteristics of an effective one-hot degree matrix for a graph. According to the Erdos-Gallai theorem, if an $N \times \mathrm{deg}{\mathrm{max}}$-dimensional degree matrix $\hat{\mathbf{L}}$ satisfies the following condition: \begin{equation}\label{eq_EG}
\sum_{i=1}^{r} \underset{j}{\arg\max\hat{\mathbf{L}}{i,j}} \leq r(r-1) + \sum_{i=r+1}^{N} \min(r, \underset{j}{\arg\max\hat{\mathbf{L}}_{i,j}}),
\end{equation}
then a valid graph can be constructed based on the degrees represented by $\hat{\mathbf{L}}$. Here, we consider a more generalized scenario. Suppose there is a degree set without a topological structure, i.e., the corresponding one-hot degree matrix, denoted as $\mathbf{L}$ , is not subject to the constraint of Eq.~\eqref{eq_EG}. If there exists a matrix $\mathbf{W}$ capable of fitting any 0-1 matrix $\mathbf{L}$, i.e.,
\begin{equation}\label{eq_MSE_condition}
  \mathrm{MSE}(\sigma(\mathbf{Z}\mathbf{W}), \mathbf{L}) < \epsilon, \forall \epsilon>0,
\end{equation}
then it is considered that $\mathbf{W}$ can fit arbitrary degree matrices. As a corollary, since the feasible domain of $\hat{\mathbf{L}}$ is a subdomain of $\mathbf{L}$, the above theory can be proven. Subsequently, by introducing Eq.~\eqref{eq_EG} as a regularization term to guide the convergence direction of $\mathbf{W}$, $\hat{\mathbf{L}}$ can be obtained.

In order to prove the solvability of $\mathbf{W}$ in Eq.~\eqref{eq_MSE_condition}, intuitively, it can be solved by solving matrix equation.
However, since $\mathbf{Z}$ is typically row full rank [4], the rank of $\mathbf{Z}^{\top}\mathbf{Z}$ is not equal to $d_z$, which implies that the pseudoinverse of $\mathbf{Z}$ can only be approximately obtained, and thus the existence of a solution to the matrix equation cannot be proven directly. To address this, we use the Moore-Penrose (MP) pseudoinverse [5] to approximate the pseudoinverse of $\mathbf{Z}$. By conducting singular value decomposition on $\mathbf{Z}$, we obtain $\mathbf{Z} = \mathbf{U}\mathbf{D}\mathbf{V}^{\top}$, where $\mathbf{U}$ and $\mathbf{V}$ are invertible matrices and $\mathbf{D}$ is the singular value matrix. The MP pseudoinverse of $\mathbf{Z}$ is solvable and given by $\mathbf{Z}^{\dagger} = \mathbf{V} \mathbf{D}^+ \mathbf{U}^{\top}$.

According to Tikhonov Regularization [6], $\mathbf{Z}^{\dagger}=\lim_{\alpha \to 0} \mathbf{Z}^{\top} (\mathbf{Z}\mathbf{Z}^{\top} + \alpha \Upsilon)^{-1}$ where $\Upsilon$ the identity matrix. We can apply the Taylor series expansion method from matrix calculus. In this case, we expand $(\mathbf{Z} \mathbf{Z}^{\top}+ \alpha \Upsilon)^{-1}$ with respect to $\alpha$. This can be expressed as:
\begin{align}\label{eq_Taylor}
\left(\mathbf{Z}\mathbf{Z}^{\top} + \alpha \Upsilon\right)^{-1} &= \left(\mathbf{Z}\mathbf{Z}^{\top}\right)^{-1} - \alpha\left(\mathbf{Z}\mathbf{Z}^{\top}\right)^{-1}\Upsilon\left(\mathbf{Z}\mathbf{Z}^{\top}\right)^{-1} \notag \\
+ O(\alpha^2) &=\left(\mathbf{Z}\mathbf{Z}^{\top}\right)^{-1}+O(\alpha),
\end{align}
where $O(\cdot)$ denotes the higher-order terms.

Therefore, for any learning target $\mathbf{Z}\mathbf{W} = \mathbf{L}$, there exists a matrix $\mathbf{W} = \mathbf{V} \mathbf{D}^+ \mathbf{U}^{\top} \mathbf{L}$, such that
\begin{align}\label{eq_Taylor_2}
\mathbf{Z}\mathbf{W}&=\mathbf{Z} \lim_{\alpha \to 0}\left(\left(\mathbf{Z}^{\top}\left(\mathbf{Z}\mathbf{Z}^{\top}\right)^{-1}+O\left(\alpha\right)\right)\right)\mathbf{L} \notag \\
&= \lim_{\alpha \to 0}\left(\left(\mathbf{Z}\mathbf{Z}^{\top}\left(\mathbf{Z}\mathbf{Z}^{\top}\right)^{-1}\mathbf{L}+\mathbf{Z}O(\alpha)\mathbf{L}\right)\right) \notag \\
&=\lim_{\alpha \to 0}\left(\mathbf{L}+O(\alpha)\right),
\end{align}
and the output loss can thus be expressed as:
\begin{equation}\label{eq_Taylor_3}
\mathrm{MSE}(\sigma(\mathbf{Z}\mathbf{W}),\mathbf{L})=\mathrm{MSE}\left(\sigma \left(\lim_{\alpha \to 0}(\mathbf{L}+O(\alpha))\right),\mathbf{L}\right)<\epsilon.
\end{equation}
The solvability of $\mathbf{W}$ as presented in Eq.~\eqref{eq_MSE_condition} is supported by Equation Eq.~\eqref{eq_Taylor_3}, which demonstrates that $\mathbf{W}$ can be fitted to any given degree matrix.

\begin{figure}[htb]
\centering
\includegraphics[width=0.48\textwidth]{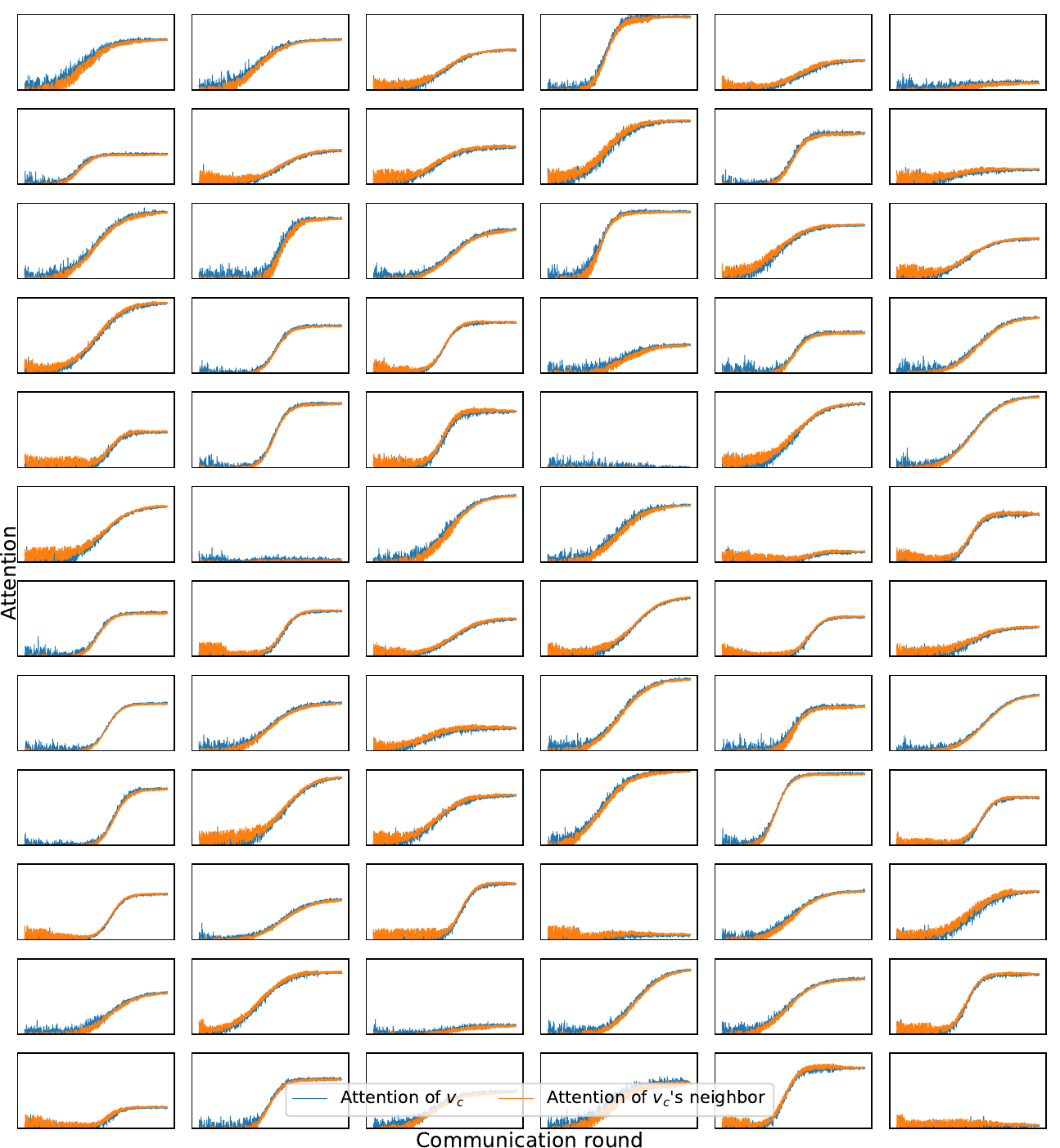}
\caption{Learning curves of the largest-degree node $v_c$ (blue lines) and its neighboring agents (orange lines) for attention on the same edge.}
\label{fig_curve_total}
\end{figure}

\begin{figure}[htb]
\centering
\includegraphics[width=0.49\textwidth]{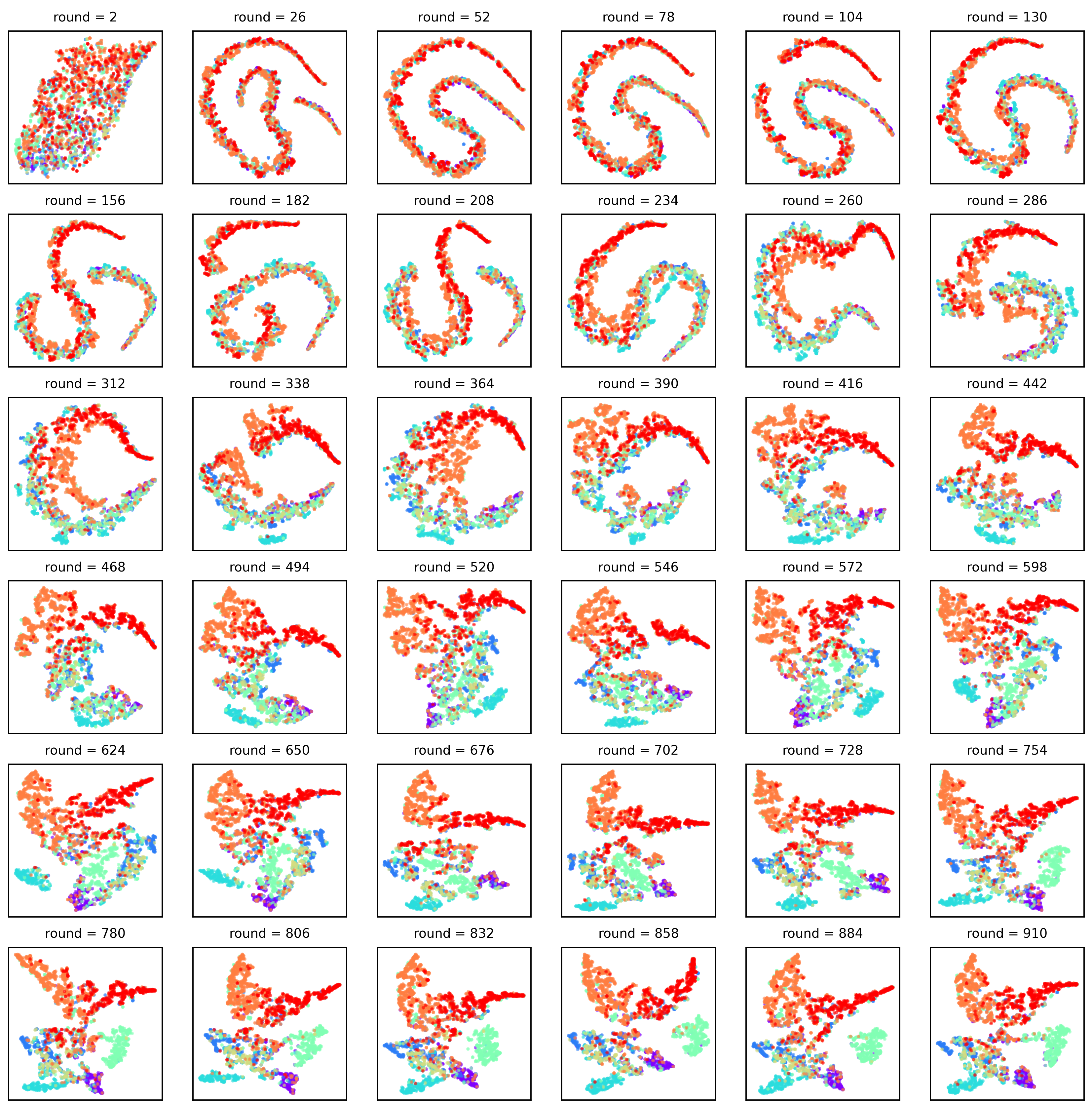}
\caption{The formation process of embedding in communication rounds.}
\label{fig_embed_total}
\end{figure}

\subsection{Proof of Theorem~\ref{thm_3}}~\label{sec_proof_T3}
First, we define two equivalent normed linear spaces, $A \sim \mathbb{R}^{d_z}$ and $B \sim \mathbb{R}^{d_z}$. The graph  $\mathcal{G}$ can be mapped into these spaces. Specifically, the node features are represented as points in spaces $A$ and $B$, while the edge distribution of $\mathcal{G}$  can be characterized as corresponding relationships between these two spaces. In this way, $\mathcal{G}$  is mapped into two spaces with corresponding relationships.

Next, we define diagonal matrices $\mathbf{D}_A=\mathrm{diag}(q_1,\ldots q_{d_z})$ and $\mathbf{D}_B=\mathrm{diag}(q_{d_z+1},\ldots q_{2d_{z}})$ for linear transformations, which, in the form of left multiplication of coordinates, scale the coordinates of vectors in spaces $A$ and $B$ across different dimensions to varying degrees. Since dimensions are mutually independent in the phenomenon space, solvable $\mathbf{D}_A$ and $\mathbf{D}_B$ exist to realize the mutual mapping of equivalent spaces.

Subsequently, we arbitrarily assign a value to $\phi$ and randomly divide it into two parts, such that $\phi$ = $\phi_1$ + $\phi_2$. We then define an equivalent space $C$ to $A$ and $B$, and sample two points with row-vector coordinates $\mathbf{c}_1$ and $\mathbf{c}_2$ in $C$ that satisfy $||\mathbf{c}_1||_1 = \phi_1 $ and $ ||\mathbf{c}_2||_1 = \phi_2$.

Now, based on the above conditions, we arbitrarily sample two points from $A$ and $B$, with row-vector coordinates $\mathbf{a}$ and $\mathbf{b}$, representing any two nodes in $\mathcal{G}$. Considering the randomness of edge distribution, $\phi$ can characterize the random neighbor relationship between them. We map $\mathbf{a}$ and $\mathbf{b}$ to space $C$ using $\mathbf{D}_A$ and $\mathbf{D}_B$, i.e., $\mathbf{c}_1=\mathbf{D}_A\mathbf{a}$, $\mathbf{c}_2=\mathbf{D}_B\mathbf{b}$. Thus, we obtain:
\begin{multline}\label{eq_Theorem3}
 \mathrm{Concat}_c(\mathbf{a},\mathbf{b}) [q_1,q_2,\ldots,q_{2d_z}]^{\top} = ||\mathbf{D}_A \mathbf{a} + \mathbf{D}_B \mathbf{b}||_1 \\
 =||\mathbf{D}_A \mathbf{a}||_1 + ||\mathbf{D}_B \mathbf{b}||_1=||\mathbf{c}_1||_1 + ||\mathbf{c}_2||_1 = \phi_1 + \phi_2 = \phi.
\end{multline}
As evidenced by Eq.\eqref{eq_Theorem3}, $\mathrm{Concat}_r(\mathbf{z}_i,\mathbf{z}_j)\mathbf{q}^{\top} = \phi$ holds.

\section{Additional Experiments}

\subsection{Symmetry of Attention in the Communication Process}\label{sec_APP_sym}
In Section~\ref{sec_sym}, we demonstrate the symmetry of well-trained attention on the same edge. Here, we further illustrate the fluctuation range of bidirectional attention on the same edge to provide additional evidence for effective communication between agents and to validate the efficacy of the neighbor feature aggregator. We select the node with the largest degree (denoted as $v_c$) in the Cora dataset as the central node and observe the forward and reverse attention of the edges formed with its 72 randomly selected neighbors. As shown in Figure~\ref{fig_curve_total}, it can be observed that the overall trends of forward and backward attention during the training process are similar, accompanied by tolerable fluctuations, and both converge to approximately the same values.

\subsection{Embedding in the Communication Process}\label{sec_emb_sym}
We provide a detailed account of the training process of node embeddings by GAgN on the Cora dataset. The experimental results are presented in Figure~\ref{fig_embed_total}. As can be seen, initially, the embeddings of all nodes are randomly dispersed in the label space. Through communication among agents, the nodes gradually reach a consensus on the embeddings, and by continuously adjusting, they gradually classify themselves correctly and achieve a steady state.

\section*{References}

\begin{itemize}\small
  \item[\lbrack 1\rbrack] Will Hamilton, Zhitao Ying, and Jure Leskovec. 2017. Inductive representation learning on large graphs. \emph{In NeurIPS}.
  \item[\lbrack 2\rbrack] Miguel Carreira-Perpinan and Weiran Wang. 2014. \emph{In AISTATS}.
  \item[\lbrack 3\rbrack] Andreas Loukas. 2019. What graph neural networks cannot learn: depth vs width. \emph{In ICLR}.
  \item[\lbrack 4\rbrack] Gama Fernando, Bruna Joan and Ribeiro Alejandro. 2020. Stability Properties of Graph Neural Networks. \emph{IEEE Trans. Signal Process.} (2020).
  \item[\lbrack 5\rbrack] Golub Gene H and Reinsch Christian. 1971. Singular value decomposition and least squares solutions. \emph{Linear algebra} (1971).
  \item[\lbrack 6\rbrack] Golub Gene H, Hansen Per Christian, O'Leary Dianne P. 1999. Tikhonov Regularization and Total Least Squares. \emph{SIAM J. Matrix Anal. Appl.} (1999)
\end{itemize}

\end{document}